\documentclass{article}

    \PassOptionsToPackage{numbers, compress}{natbib}

\usepackage[main, final]{neurips_2026}
\makeatletter
\renewcommand{\@notice}{}
\makeatother

\usepackage[utf8]{inputenc} %
\usepackage[T1]{fontenc}    %
\usepackage{url}            %
\usepackage{booktabs}       %
\usepackage{amsfonts}       %
\usepackage{nicefrac}       %
\usepackage{microtype}      %
\usepackage{xcolor}         %
\usepackage{graphicx}
\usepackage{amsmath}
\usepackage{amsthm}
\usepackage{subcaption}
\usepackage{float}
\usepackage{algorithm}
\usepackage{algorithmic}
\usepackage{wrapfig}
\usepackage{mathtools}
\usepackage{basic_commands}
\newtheorem{lemma}{Lemma}
\newtheorem{theorem}{Theorem}
\newtheorem{remark}{Remark}
\definecolor{myred}{HTML}{D21418}
\usepackage[
  colorlinks=true,
  citecolor=myred,
  linkcolor=myred,
  urlcolor=myred
]{hyperref}
\usepackage[capitalise,noabbrev,nameinlink]{cleveref}

\newcommand{\pl}[1]{#1}
\newcommand{\kldiv}[2]{\KL\!\left(#1\,\|\,#2\right)}

\title{Diffusion Guidance Is a Controllable \\ Policy Improvement Operator}

\author{Kevin Frans$^*$ \\
  UC Berkeley\\
  \texttt{kvfrans@berkeley.edu} \\
  \And
  Seohong Park$^*$ \\
  UC Berkeley \\
  \And
  Pieter Abbeel \\
  UC Berkeley \\
  \And
  Sergey Levine \\
  UC Berkeley \\
}

\begin{document}

\maketitle

\begin{abstract}
At the core of reinforcement learning is the idea of learning beyond the performance in the data. However, scaling such systems has proven notoriously tricky. In contrast, techniques from generative modeling have proven remarkably scalable and are simple to train. 
In this work, we combine these strengths, by deriving a direct relation between policy improvement and guidance of diffusion models. The resulting framework, CFGRL, is trained with the simplicity of supervised learning, yet can further improve on the policies in the data.
On offline RL tasks, we observe a reliable trend---increased guidance weighting leads to increased performance. Of particular importance, CFGRL can operate without explicitly learning a value function, allowing us to generalize simple supervised methods (e.g., goal-conditioned behavioral cloning) to further prioritize optimality, gaining performance for ``free'' across the board.
\end{abstract}

\section{Introduction}

\def\thefootnote{*}\footnotetext{Equal contribution.}\def\thefootnote{\arabic{footnote}}

Reinforcement learning (RL) provides a powerful framework for autonomous agents to attain strong performance by directly optimizing for task rewards. However, scaling up RL algorithms has proven notoriously challenging, particularly when using off-policy datasets. In contrast, modern generative modeling techniques have proven remarkably scalable, and have been used for related problems such as behavioral cloning \citep{diffuser_janner2022, pi0_black2024, diffusionpolicy_chi2023}. Can we leverage expressive generative modeling tools to derive simple, scalable RL approaches?

At the heart of RL is the idea of optimizing \textit{beyond the performance shown in the data.} This is especially important when training agents from offline data that may have been collected by an exploratory or otherwise suboptimal policy. On one end of the spectrum, behavioral cloning methods are simple and can leverage stable generative modeling tools like diffusion \cite{ddpm_ho2020, lipman2024flow} and flow-matching, but are only as optimal as the data. On the other end, iterative RL techniques are in principle more optimal, but in practice can suffer from hyperparameter sensitivity and instability \citep{bottleneck_park2024, cql_kumar2020, td3_fujimoto2018} that have made it challenging to scale to larger tasks.

In this work, we combine the strengths of both settings by developing a framework which is trained with the simplicity of behavioral cloning, yet can further improve on the data behaviors. We first define policies as products of two factors -- a prior reference policy, and an ``optimality'' distribution. When the optimality distribution is proportional to a monotonically increasing function of advantage, we prove that the resulting product will be an improvement over the prior. 

The key insight is that we can sample from this product distribution via techniques from diffusion modeling, and we can do so in a straightforward and controllable way. Rather than optimizing an optimality \textit{predictor}, we instead learn an equivalent optimality-\textit{conditioned} policy, as done in classifier-free guidance \citep{cfg_ho2022}. The prior and conditional factors can then be dynamically combined during sampling, allowing for a degree of policy improvement that can be controlled \textit{during test time}, without the need for retraining. 

Our framework, which we refer to as CFGRL, provides a principled and powerful connection between generative modeling and RL. Diffusion and flow-matching policies already represent some of the most powerful approaches for imitation learning, but typically do not make use of guidance for policy learning \citep{pi0_black2024, diffusionpolicy_chi2023}. CFGRL bridges a connection between guidance and traditional RL objectives---in fact, under certain choices, guided sampling results in a distribution that is equivalent to the solution of a KL-constrained policy improvement objective.

We can instantiate CFGRL in a number of ways. As a tool for offline policy extraction, CFGRL can be combined with an off-the-shelf value learning method to provide a simple way to train diffusion policies that maximize reward. CFGRL provides a consistent improvement over standard baselines such as weighted regression methods and direct reparameterized gradients, both using Gaussian and diffusion policies. Scaling trends highlight how increasing the CFGRL guidance term results in stronger policies up to a divergence point. Importantly, whereas most other methods require tuning a regularization term with many training runs,
CFGRL requires only a \textit{single} training run, as regularization can be tuned entirely at test time.

A particularly simple yet powerful setting for CFGRL is its use as an entirely \textit{value-free} learning method for goal-conditioned behavior cloning. In this setting, CFGRL \textit{only} requires training a standard BC policy, and does not utilize any value functions or additional components. Nevertheless, CFGRL unlocks further policy improvement beyond the offline data, and reliably outperforms goal-conditioned behavioral cloning baselines across the board on state-based, visual, and hierarchical settings, at times increasing success rates by a factor of two.

Our main contributions are twofold. First, we propose a \textit{principled connection} between diffusion model guidance and policy improvement in reinforcement learning. Second, we develop a set of simple \textit{practical algorithms} that utilize the above connection to reliably improve policies in both the offline RL and goal-conditioned behavioral cloning settings.

The CFGRL code is available at \url{https://github.com/kvfrans/cfgrl}.

\section{Related work}

\textbf{Offline RL.} Unlike standard RL, which involves exploring an environment, offline RL aims to learn a reward-maximizing policy solely from a previously collected dataset. The key challenge is to improve performance while preventing erroneous extrapolation when deviating too far from the dataset. Previous works target this problem from the value learning \citep{cql_kumar2020, iql_kostrikov2022, sql_xu2023, xql_garg2023, edac_an2021, sacrnd_nikulin2023}, and policy extraction \citep{brac_wu2019, td3bc_fujimoto2021, rebrac_tarasov2023, fql_park2025, awr_peng2019, awac_nair2020, sfbc_chen2023, idql_hansenestruch2023} directions. Our work most closely relates to weighted regression \citep{awr_peng2019, awac_nair2020} and return-conditioned behavioral cloning \citep{rcp_kumar2019, dt_chen2021, qdt_yamagata2023} methods (of which goal-conditioned hindsight relabeling \citep{her_andrychowicz2017, gcsl_ghosh2021, rvs_emmons2022, ocbc_eysenbach2022} is a special case), due to their emphasis on simple supervised objectives. However, we instead frame the tradeoff between regularization and policy improvement in terms of guiding a diffusion model, providing a way to \textit{control} this tradeoff during test time.

\textbf{Diffusion and flow policies for RL.}
Previous works have proposed diverse ways
to leverage the expressivity of iterative generative models,
such as diffusion~\citep{diffusion_sohl2015, ddpm_ho2020}
and flow models~\citep{flow_lipman2023, flow_liu2023, flow_albergo2023},
to enhance the capabilities of RL policies.
The main challenge with diffusion policy learning lies in \emph{how to extract}~\citep{fql_park2025}:
a diffusion policy to maximize the learned Q-function.
Prior works propose strategies
based on weighted regression~\citep{qgpo_lu2023, edp_kang2023, qvpo_ding2024, qipo_zhang2025},
reparameterized gradients~\citep{dql_wang2023, diffcps_he2023, consistencyac_ding2024, srdp_ada2024, entropydql_zhang2024, fql_park2025},
rejection sampling~\citep{sfbc_chen2023, idql_hansenestruch2023, aligniql_he2024},
and more~\citep{dipo_yang2023, acd_kuba2023, diffusiondice_mao2024, eda_chen2024, srpo_chen2024, qsm_psenka2024, dtql_chen2024, ddiffpg_li2024, parl_mark2024, dac_fang2025, dppo_ren2025}.
Our method introduces classifier-free guidance as a policy extraction mechanism.
This has multiple benefits over previous approaches:
unlike reparameterized gradient-based methods, it does not require (potentially unstable) backpropagation through time~\citep{fql_park2025};
unlike rejection sampling, it does not involve a costly sampling-then-filtering procedure.
Algorithmically, the works closest to ours are \citet{dd_ajay2023} and \citet{acd_kuba2023},
which similarly utilize classifier-free guidance to improve sequential decision making.
However, unlike \citet{dd_ajay2023}, we do not require modeling the entire state-action sequences nor explicit return estimates (note that their model is around $100\times$ larger than the policies we consider). Unlike \citet{acd_kuba2023}, our framework supports a range of optimality functions rather than only $A=0$, and does not rely on further rejection sampling. 
Additionally, CFGRL can be instantiated flexibly to work in value-free settings ---
in \Cref{section:gcbc}, we further improve a goal-conditioned diffusion policy \textit{without training any value functions},
whereas all aforementioned techniques would require doing so.

\section{Preliminaries}

In the standard reinforcement learning (RL) setting, we define an environment as a Markov decision process with state space $\gS$, action space $\gA$, a transition function $p(\pl{s'} \mid \pl{s}, \pl{a}): \gS \times \gA \to \Delta(\gS)$, reward function $r(\pl{s}): \gS \to \sR$, and initial state distribution $p(\pl{s_0}) \in \Delta(\gS)$. We assume the state is fully observable. An agent is defined by a policy $\pi(\pl{a} \mid \pl{s}): \gS \to \Delta(A)$ that represents a probability distribution over actions, and together with the environment, produces a distribution of state-action trajectories $\tau = (s_0, a_0, s_1, a_1, \ldots)$. The standard RL objective is to learn a parameterized policy $\pi_\theta$ that maximizes the expected sum of future discounted rewards along such trajectories:
\begin{equation}
    J(\pi_\theta) = \mathbb{E}_{\tau \sim p(\pl{\tau} \mid \pi_\theta)}  \sum_t \gamma^t r(s_t, a_t),
    \label{equation:rl}
\end{equation}
where $p(\tau \mid \pi_\theta)$ is defined as $p(s_0) \prod_{t=0}^\infty p(s_{t+1} \mid s_t, a_t) \pi_\theta(a_t \mid s_t)$.

A \textit{policy improvement operator} is an update from
a reference policy $\hat{\pi}$ to a new policy $\pi$
such that the RL objective above does not decrease:
$J(\hat \pi) \leq J(\pi)$.
This concept is formalized via $V_{\hat{\pi}}(s)$ and $Q_{\hat{\pi}}(s,a)$, which denote the discounted expected future reward under the reference policy starting from a given state or state-action pair, respectively~\citep{rl_sutton2005}. The difference between these terms is the advantage, $A_{\hat{\pi}}(s,a) = Q_{\hat{\pi}}(s, a) - V_{\hat{\pi}}(s)$. A classic result shows that any update with non-negative advantage under the resulting state distribution, such that
\begin{equation}
   \E_{(s, a) \sim p_\pi(\pl{s}, \pl{a})} [A_{\hat{\pi}}(s,a)] \geq 0,
   \label{eq:advantage}
\end{equation}
with $p_\pi(\pl{s}, \pl{a})$ denoting the discounted state-action occupancy distribution,
results in policy improvement~\citep{trpo_schulman2015}. In practice, we require algorithms that operate over samples from a previous reference policy. Therefore, practical algorithms will approximate the above condition under the previous policy's state distribution $p_{\hat{\pi}}(\pl{s})$, and aim to maximize:
\begin{equation}
   \tilde{J}(\pi) =\E_{s \sim p_{\hat \pi}(\pl{s})}[ \E_{a \sim \pi(\pl{a} \mid s)}[ A_{\hat{\pi}}(s,a) ]].
   \label{eq:improvement}
\end{equation}
Prior work has shown that, as long as the divergence between the reference and resulting policy is bounded, the approximate objective in \cref{eq:improvement} provides a bound on the true objective in \cref{eq:advantage}, enabling monotonic incremental improvement~\citep{trpo_schulman2015}.

To account for this divergence, it is common to utilize trust-region methods during policy improvement~\citep{trpo_schulman2015, ppo_schulman2017}. While various divergence measures have been considered~\citep{dualrl_sikchi2024}, a standard choice is to penalize the KL divergence between the reference and resulting policy, resulting in the following KL-penalized RL objective as parameterized by a constant $\beta$:
\begin{equation}
    J(\pi_\theta) = \mathbb{E}_{\tau \sim p(\pl{\tau} \mid \pi_\theta)} \left[ \sum_t \gamma^t r(s_t, a_t) \right] - \beta \mathbb{E}_{s \sim p_{\pi}(\pl{s})} \left[ \kldiv{\pi_\theta(\pl{a} \mid s)}{\hat{\pi}(\pl{a} \mid s)} \right].
    \label{equation:klrl}
\end{equation}

One strategy to optimize the above objective is via iteratively applying a policy gradient \citep{pg_sutton1999}; however, such methods require on-policy samples and can have high variance. The community has instead often relied on methods that resemble supervised learning, such as weighted regression \citep{rwr_peters2007, awr_peng2019}. 
In the following sections, we introduce a policy improvement strategy based on generative modeling that maintains the simplicity of supervised learning, yet allows for controllable policy improvement.

\section{Diffusion guidance is a controllable policy improvement operator}
\label{section:operator}

\begin{figure}
    \centering
    \includegraphics[width=\textwidth]{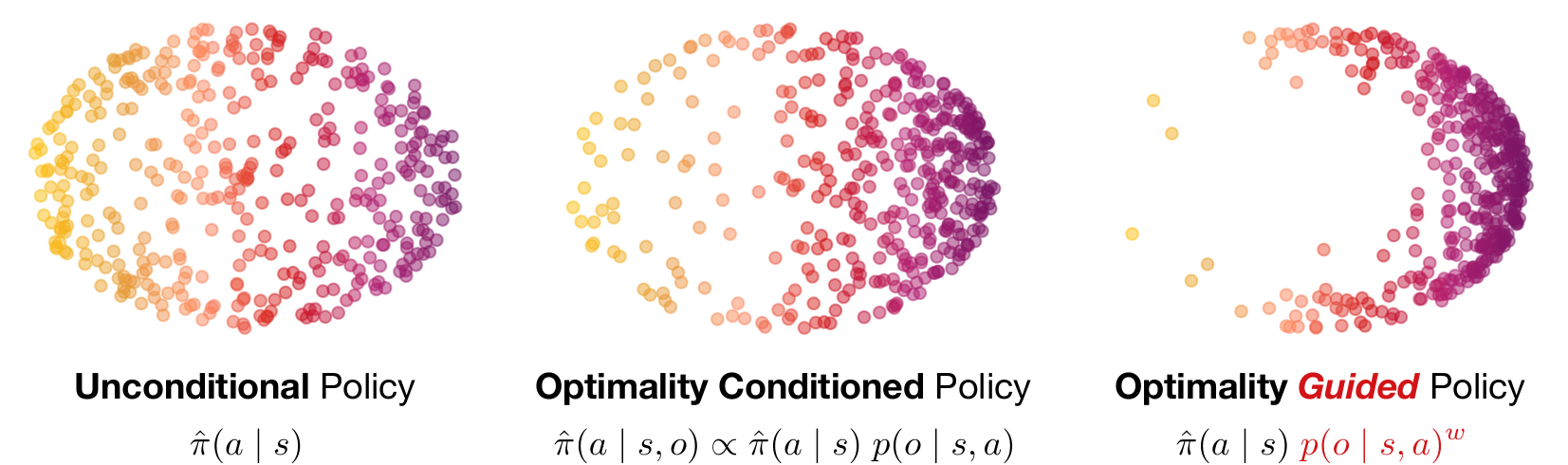}
    
    \caption{\footnotesize
    \textbf{While conditioning on optimality can create a baseline level of improvement, policies can be \textit{further} improved by attenuating this conditioning.} When $p(o \mid s,a)$ is proportional to a monotonically increasing function of advantage, then attenuation can increase expected return, and this can be accomplished naturally with diffusion guidance.}
    \label{fig:teaser}
\end{figure}

In this work, we establish a connection between classifier-free guidance in diffusion modeling and policy improvement in RL. We then use this relation to develop a simple framework, CFGRL, which enables us to leverage stable and scalable diffusion model training methods to enable policies that can improve over the performance in the dataset. Specifically, CFGRL allows the degree of policy improvement to be controlled at \textit{test time}, rather than having to be decided during training. 
Additionally, CFGRL does not necessarily require the use of an explicit Q-function, allowing it to be used as a drop-in replacement that further improves on methods such as goal-conditioned behavioral cloning.

\textbf{Product policies.} We begin by parameterizing policies as a product of two factors---a reference policy, and an \textit{optimality} function $f: \mathbb{R} \rightarrow \mathbb{R}$, which is conditional on advantage:
\begin{equation}
    \pi(a \mid s) \propto \hat{\pi}(a \mid s)\; f(A(s,a)).
    \label{eq:product-policy}
\end{equation}
The motivation behind this factorization is to frame improved policies in terms of a probabilistic adjustment to the current reference policy. As we will show later, diffusion guidance is naturally suited to sampling from such distributions.

We now show how product policies can be used as \textit{policy improvement operators.} When $f$ fulfills certain common criteria, the resulting policy will provably be an improvement over the reference:
\begin{remark}[\textbf{Improvement of product policies}]
If $f$ is a non-negative, monotonically increasing function of $A_{\hat{\pi}}(s,a)$, then the product $\pi(a \mid s) \propto \hat{\pi}(a \mid s) f(A(s,a))$ is an improvement over $\hat{\pi}(a \mid s)$.
\label{thm:product_improvement}
\end{remark}
\vspace{-10pt}
We provide a formalization and proof of this claim in the Appendix (\Cref{thm:reweight}),
which generalizes previous results with bandits \citep{em_dayan1997, rwr_peters2007} and exponentiated advantages \citep{awr_peng2019}.

The above finding reveals a simple path towards policy improvement. If we can sample from a properly weighted product policy, then the resulting policy will achieve a higher (or at least equal) expected return than the reference. 

Empirically, we can control the \textit{degree} of improvement by sampling from an attenuated optimality function.
To do so, we consider a product policy where the optimality function is exponentiated:
\begin{align}
    \pi_{w}(a \mid s) \propto \hat{\pi}(a \mid s) \;f(A(s,a))^{w}
\end{align}
for some $w \ge 0$.
By \Cref{thm:reweight}, any such policy of this form is an improvement over $\hat \pi(a \mid s)$.
In practice, there is a trade-off with $w$:
a higher exponent may lead to better performance,
but it can also make the policy deviate further from the reference policy.
It provides a convenient, controllable knob to adjust the strength of optimality weighting at test time.

This tradeoff between adhering to the reference policy and maximizing return can be understood via the KL-regularized RL objective in \cref{equation:klrl}. Notably, the \textit{solutions} to the mentioned objective naturally form a set of product policies:
\begin{remark}[\textbf{KL-regularized reward-maximization results in product policies}~\citep{awr_peng2019}]
    The policies that maximize \cref{equation:klrl} under a given KL-penalty $\beta$ take the form of
    \begin{align}
        \pi(a \mid s) \propto \hat{\pi}(a \mid s) \exp(A(s,a))^{1/\beta}.
    \end{align}
\end{remark}

Equivalently, the $\beta$ term can be folded into the exponent, e.g. $\pi(a \mid s)\propto \hat{\pi}(a \mid s)\exp(A(s,a)/\beta)$. This condensed objective has been used in prior works \citep{rwr_peters2007, awr_peng2019} to directly learn $\pi$; however, in their methods the $\beta$ hyperparameter must be specified ahead of time. As shown in the next sections, we will instead develop a framework where the product factors are represented \textit{independently}, allowing their composition to be freely controlled during evaluation time.

\subsection{Composing factors via diffusion guidance}

Having understood that product policies are a natural way to induce policy improvement, we can now instantiate such policies using machinery from diffusion modeling. We start by casting the optimality function as a binary random variable\footnote{Slightly abusing notation, we abbreviate $o=1$ as $o$ when it is clear from context. $o=\varnothing$ represents an unconditional case where $o$ is not specified, i.e., the union of $o=0$ and $o=1$.} $o \in \{\varnothing,0, 1\}$ whose likelihood is defined via $f$:
\begin{equation}
    p(o \mid s,a) = f(A(s,a)) / Z(s)
\end{equation}
where $Z(s) = \int f(A(s,a'))da'$ is a state-dependent normalization factor. We will not require estimating $Z(s)$. The product policy from \cref{eq:product-policy} can now be equivalently defined as:
\begin{equation}
    \pi(a \mid s) \propto \hat{\pi}(a \mid s)\; p(o \mid s,a).
\end{equation}
Recall that diffusion models implicitly model a distribution by learning its normalized \textit{score function}, i.e., the gradient of log-likelihood under that distribution \citep{ncsn_song2019}.
Score functions have the useful property that for product distributions, they are additively composable. As such, the score of the product policy above can be represented as the sum of two factors:
\begin{equation}
    \nabla_a \log \pi(a \mid s) = \nabla_a \log \hat{\pi}(a \mid s) + \nabla_a \; \log p(o \mid s,a).
    \label{eq:factors}
\end{equation}

\textbf{Avoiding an explicit optimality predictor.} In many cases, we would rather avoid explicitly learning the $p(o \mid s,a)$ distribution. For one thing, optimality must hold a valid probability distribution, and calculating the normalization term $Z(s)$ can be tricky. Secondly, explicitly backpropagating through a neural network predictor may result in adversarial gradient attacks \citep{adversarial_goodfellow2015} especially at out-of-distribution actions \citep{iql_kostrikov2022, cql_kumar2020}. We also note that if one wanted to learn an optimality predictor, it would need to remain accurate under the support of \textit{partially-noised} actions.

Instead, we can utilize an insight from classifier-free guidance \citep{cfg_ho2022}, and use Bayes' rule to invert the optimality distribution into an optimality-\textit{conditioned} policy score function:
\begin{equation}
    \nabla_a \log \pi(a \mid s) = \nabla_a \log \hat{\pi}(a \mid s) +  (\nabla_a \log\hat{\pi}(a \mid s,o) - \nabla_a \log \hat{\pi}(a \mid s)).
    \label{eq:inversion}
\end{equation}
With the above form, both factors can be unified into a single conditional policy. We can represent both factors with the same neural network, and train via a straightforward diffusion modeling objective.

\textbf{Guidance controls the attenuation of optimality.} A key benefit of defining the product distribution in terms of composable factors is that the ratio between these factors can be dynamically controlled. Introducing a guidance weight $w$, the score function
\begin{equation}
    \nabla_a \log \hat{\pi}(a \mid s) +  w \;(\nabla_a \log\hat{\pi}(a \mid s,o) - \nabla_a \log \hat{\pi}(a \mid s)) 
\end{equation}
corresponds to the attenuated distribution
\begin{equation}
    \pi(a \mid s) \propto \hat{\pi}(a \mid s) \; p(o \mid s,a)^w, \; \text{and equivalently} \; \; \hat{\pi}(a \mid s)  f(A(s,a))^w.
\end{equation}
By controlling the guidance weight during sampling, we can sample from product policies that improve on the reference policy to different degrees (at the cost of adherence to the prior).
The tradeoff is a key hyperparameter to tune in offline RL \citep{bottleneck_park2024}, and often requires many sweeps. In contrast, with CFGRL, this sweep can be performed at test-time over a \textit{single} network, without the need for retraining.

\vspace*{-4pt}
\subsection{Training and sampling with CFGRL}
\vspace*{-4pt}
\label{section:training}

\begin{figure}
\vspace*{-4ex}
\begin{minipage}[t]{0.45\textwidth}
\begin{algorithm}[H]
\caption{CFGRL Training}
\begin{algorithmic}
   \WHILE{not converged}
        \STATE \text{Collect data, or use offline data.}
        \STATE $(s,a) \sim D$, $\textstyle a_0 \sim \mathcal{N}(0,I)$, $t \sim U(0,1)$
        \STATE \text{Label with optimality $o\in \{0, 1\}$}.
        \STATE \text{If $\mathrm{rand}() < 0.1$, set optimality $o=\varnothing$.}
        \STATE $a_t \leftarrow (1-t)\,a_0 + t\,a$
        \STATE $\textstyle\theta \leftarrow \theta - \eta \nabla_\theta \| v_\theta(a_t, t, s, o) - (a - a_0) \|^2$
   \ENDWHILE
\end{algorithmic}
\end{algorithm}
\end{minipage}
\hfill
\begin{minipage}[t]{0.55\textwidth}
\begin{algorithm}[H]
\caption{CFGRL Sampling}
\begin{algorithmic} %
    \STATE $a \sim \mathcal{N}(0,I)$
    \STATE $t \leftarrow 0$
    \FOR{$n \in [0, \dots ,N-1]$}
        \STATE $v = (1-w) \, v_\theta(a, t, s, \varnothing) + w \, v_\theta(a, t, s, o=1)$
        \STATE $a \leftarrow a + (1/N)v$
        \STATE $t \leftarrow t + (1/N)$
    \ENDFOR  %
    \STATE \textbf{return} $a$
\end{algorithmic}
\label{alg:sampling}
\end{algorithm}
\end{minipage}
\end{figure}

We instantiate a single diffusion network to serve as both the conditional and unconditional policy. For simplicity, we adopt the flow-matching framework for training. While flow networks predict velocity rather than score,
there exists an exact linear correspondence between the optimal velocity and score,
which allows us to apply the same CFG framework to flow matching as well~\citep{guidedflow_zheng2023, lipman2024flow}.
The policy is modeled by a velocity field $v_\theta$ conditioned on a partially-noised action $a_t$, along with a noise scale $t$, the current state $s$, as well as the previously defined \textit{optimality variable} $o \in \{\varnothing, 0, 1\}$. This network is trained via the following loss function:
\begin{equation}
    \mathcal{L}(\theta) = \mathbb{E}_{s,a \sim D} \left[ \| v_\theta(a_t, t, s,o) - (a-a_0) \|^2 \right] \qquad \text{where} \qquad a_t = (1-t)a_0 + ta,
\end{equation}
and where the noise scale $t$ is sampled uniformly between $[0,1]$, and $a_0 \sim N(0,1)$ is Gaussian noise.

\section{CFGRL is a strong test-time policy extraction method for offline RL}
\label{section:offline}

A common approach to offline RL is to learn a state-action value function $Q_\theta(s,a)$, and then extract a policy from this value function via a regularized policy extraction method, where the policy is trained to maximize the value function while staying close to the behavior policy. This regularization is critical to avoid out-of-distribution actions for which the value function is likely to be overestimated~\citep{awr_peng2019, iql_kostrikov2022}. 
For many policy extraction methods, such as advantage-weighted regression (AWR) and its variants \citep{awr_peng2019, awac_nair2020, crr_wang2020}, re-parameterized gradient methods such as DDPG \citep{ddpg_lillicrap2016}, or flow Q-learning (FQL) \citep{fql_park2025}, the weight on the policy regularizer is a highly sensitive hyperparameter that must be tuned carefully at training time~\citep{bottleneck_park2024}, requiring a costly hyperparameter sweep and, in realistic settings, a reliable offline evaluation method or validation task. In contrast, we now show how CFGRL can be used as a drop-in replacement to avoid this train-time sweep entirely, requiring only a \textit{single} training run and using the guidance weight as a \emph{test-time} hyperparameter.

As we assume access to value function, we instantiate CFGRL with a very simple optimality criterion:
\begin{equation}
    o = \begin{cases}
      1 & \text{if } A(s,a) \geq 0 \\
      0 & \text{if } A(s,a) < 0,
    \end{cases}
\end{equation}
which nicely lends itself to Bayes' inversion for \cref{eq:inversion}. Equivalently, we are assigning $f = \mathbf{1}(A \geq0)$, which is both non-negative and non-decreasing, fulfilling the criteria of \cref{thm:product_improvement}.
Here, while $f(A)^w = f(A)$ holds for all $w > 0$ and thus the idealized product policy does not depend on $w$ in theory,
we found that varying $w$ with this simple criterion can still improve performance in practice.
While we choose to use this version for simplicity,
we note that it can be replaced with a different gating function to have the product policy depend properly on $w$.

\begin{figure}[t!]
    \centering
    \includegraphics[width=\textwidth]{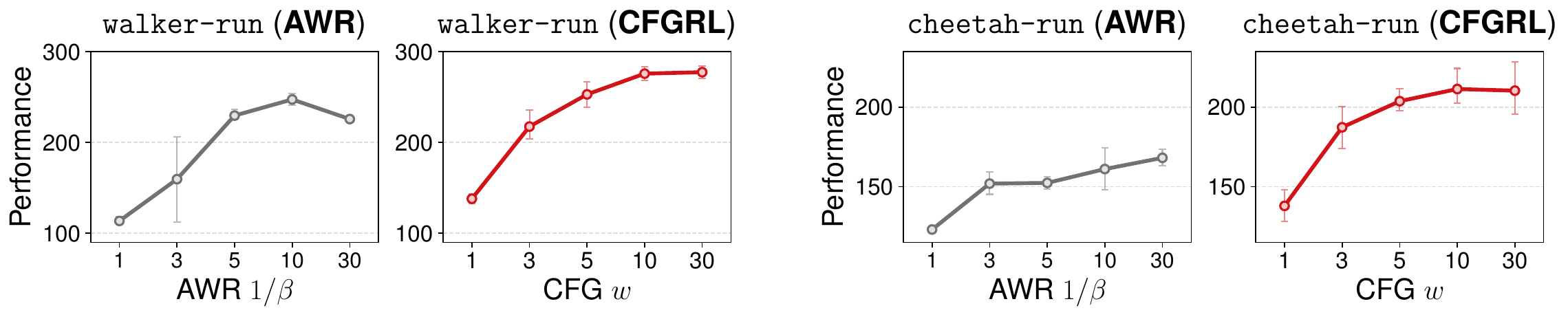}
    \vspace*{-3ex}
    \caption{\footnotesize
    \textbf{CFGRL controls the tradeoff between reference adherence and optimality by adjusting the guidance weighting.} This addresses the same motivation behind tuning the temperature in advantage-weighted regression and tuning the BC coefficient in reparameterized policy gradients. However, it can be tuned during \textit{test time} rather than via retraining, and empirically leads to a higher maximum performance.}
    \label{fig:temperature}
\end{figure}

\clearpage

\begin{wrapfigure}{r}{0.45\textwidth}
    \centering
    \vspace{10pt}
    \raisebox{0pt}[\dimexpr\height-1.0\baselineskip\relax]{
        \begin{subfigure}[t]{1.0\linewidth}
        \includegraphics[width=\linewidth]{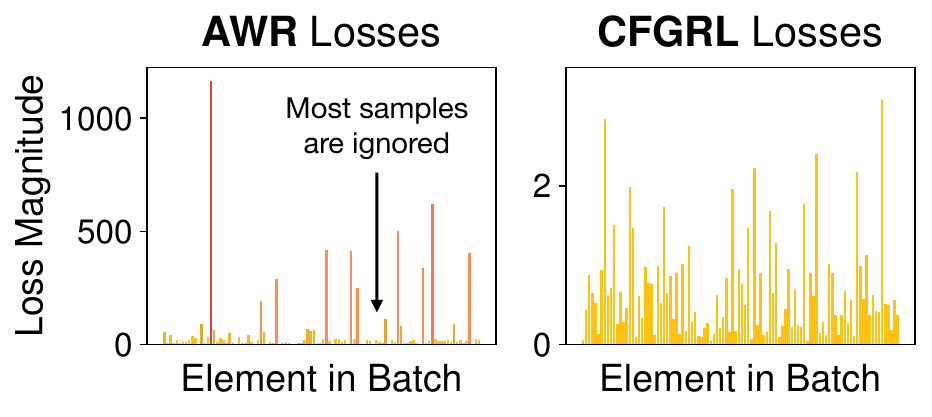}
        \end{subfigure}
    }
    \vspace{-10pt}
    \caption{
    \footnotesize
    \textbf{Weighted regression methods result in uneven gradient magnitudes within a batch.} This can limit the effective signal that each batch provides. In contrast, CFGRL uses a simple conditional diffusion modeling loss with even weighting.
    }
    \label{fig:awr}
\end{wrapfigure}
Similarly to AWR, we train policies only on state-action pairs in the dataset. 
Given a state-action sample $(s,a) \sim D$, we label it with $o \in \{0,1\}$ according to the above criterion. We then train with the standard conditional diffusion modeling loss as described in \cref{section:training}. Notably, unlike AWR, there is \textit{no importance weighting term used in training.} As such, gradients within each batch remain reasonably distributed, as shown in \cref{fig:awr} (right).

\begin{table}[t!]

\centering
\vspace{-10pt}
\captionof{table}{
\footnotesize
\textbf{Offline RL results.}
Numbers at or above 95\% of the best mean for each task are boldfaced.
}\label{table:aggregation}
\vspace{5pt}
\scalebox{0.8}
{
\begin{tabular}{lccccc}
\toprule
\texttt{Task} & \texttt{AWR} & \texttt{Diffusion AWR} & \texttt{FQL} & \texttt{DDPG} & \texttt{\color{myred}CFGRL} \\
\midrule

\texttt{pointmaze-large-navigate} & $41.7$ {\tiny $\pm 43.3$} & $\mathbf{98.3}$ {\tiny $\pm 2.9$} & $\mathbf{98.3}$ {\tiny $\pm 2.9$} & $0.0$ {\tiny $\pm 0.0$} & $75.0$ {\tiny $\pm 43.3$} \\
\texttt{pointmaze-teleport-navigate} & $0.0$ {\tiny $\pm 0.0$} & $35.0$ {\tiny $\pm 15.9$} & $0.8$ {\tiny $\pm 1.4$} & $0.0$ {\tiny $\pm 0.0$} & $\mathbf{49.2}$ {\tiny $\pm 15.9$} \\
\texttt{antmaze-large-navigate} & $\mathbf{45.0}$ {\tiny $\pm 27.3$} & $15.0$ {\tiny $\pm 8.3$} & $19.2$ {\tiny $\pm 29.4$} & $1.7$ {\tiny $\pm 1.7$} & $\mathbf{45.0}$ {\tiny $\pm 23.4$} \\
\texttt{antmaze-teleport-navigate} & $10.0$ {\tiny $\pm 7.5$} & $5.0$ {\tiny $\pm 5.5$} & $\mathbf{37.5}$ {\tiny $\pm 8.6$} & $3.3$ {\tiny $\pm 4.1$} & $19.2$ {\tiny $\pm 6.0$} \\
\texttt{humanoidmaze-large-navigate} & $5.8$ {\tiny $\pm 4.3$} & $0.8$ {\tiny $\pm 1.4$} & $0.0$ {\tiny $\pm 0.0$} & $0.0$ {\tiny $\pm 0.0$} & $\mathbf{8.3}$ {\tiny $\pm 2.9$} \\
\texttt{antsoccer-arena-navigate} & $8.3$ {\tiny $\pm 5.0$} & $14.2$ {\tiny $\pm 4.9$} & $31.7$ {\tiny $\pm 7.6$} & $1.7$ {\tiny $\pm 1.7$} & $\mathbf{41.7}$ {\tiny $\pm 6.9$} \\
\texttt{cube-single-play} & $40.0$ {\tiny $\pm 19.3$} & $40.0$ {\tiny $\pm 17.8$} & $\mathbf{56.7}$ {\tiny $\pm 14.7$} & $8.3$ {\tiny $\pm 5.5$} & $45.8$ {\tiny $\pm 23.7$} \\
\texttt{scene-play} & $39.2$ {\tiny $\pm 12.6$} & $35.8$ {\tiny $\pm 9.2$} & $0.0$ {\tiny $\pm 0.0$} & $5.8$ {\tiny $\pm 4.3$} & $\mathbf{44.2}$ {\tiny $\pm 11.1$} \\
\texttt{puzzle-3x3-play} & $1.7$ {\tiny $\pm 1.7$} & $\mathbf{5.8}$ {\tiny $\pm 2.8$} & $0.0$ {\tiny $\pm 0.0$} & $0.8$ {\tiny $\pm 1.4$} & $2.5$ {\tiny $\pm 1.4$} \\

\midrule

\texttt{square-ph} & -- & $15.6$ {\tiny $\pm 5.7$} & $4.4$ {\tiny $\pm 4.2$} & -- & $\mathbf{50.0}$ {\tiny $\pm 7.2$} \\
\texttt{square-mh} & -- & $3.3$ {\tiny $\pm 2.7$} & $1.1$ {\tiny $\pm 1.6$} & -- & $\mathbf{16.7}$ {\tiny $\pm 2.7$} \\
\texttt{can-ph} & -- & $\mathbf{95.6}$ {\tiny $\pm 6.3$} & $84.4$ {\tiny $\pm 1.6$} & -- & $\mathbf{98.9}$ {\tiny $\pm 1.6$} \\
\texttt{can-mh} & -- & $64.4$ {\tiny $\pm 6.3$} & $41.1$ {\tiny $\pm 17.5$} & -- & $\mathbf{77.8}$ {\tiny $\pm 9.6$} \\
\texttt{lift-ph} & -- & $\mathbf{100.0}$ {\tiny $\pm 0.0$} & $91.1$ {\tiny $\pm 10.3$} & -- & $\mathbf{100.0}$ {\tiny $\pm 0.0$} \\
\texttt{lift-mh} & -- & $80.0$ {\tiny $\pm 8.2$} & $80.0$ {\tiny $\pm 12.5$} & -- & $\mathbf{97.8}$ {\tiny $\pm 1.6$} \\
\texttt{transport-ph} & -- & $5.6$ {\tiny $\pm 1.6$} & $2.2$ {\tiny $\pm 1.6$} & -- & $\mathbf{6.7}$ {\tiny $\pm 0.0$} \\
\texttt{transport-mh} & -- & $0.0$ {\tiny $\pm 0.0$} & $0.0$ {\tiny $\pm 0.0$} & -- & $\mathbf{2.2}$ {\tiny $\pm 1.6$} \\

\bottomrule
\end{tabular}
}
\end{table}

There is a suggestive similarity between the guidance weight in CFGRL, the temperature hyperparameter in AWR, and the BC-regularization terms in DDPG and FQL. In fact, all these parameters play the same role---they control the tradeoff between adherence to a reference policy and maximization of rewards (\cref{fig:temperature}; evaluated on ExORL~\citep{exorl_yarats2022}). However, with AWR, DDPG, or FQL, the temperature must be chosen beforehand and used during training, requiring a large sweep over the values. CFGRL keeps the prior policy and the optimality-conditioned policy separate, and only combines them during sampling. Thus, this tradeoff can be adjusted \textit{without retraining} when using CFGRL, making it easy to find the best value.

\textbf{Experimental comparison.} We further establish the comparison between CFGRL and baseline methods 
on $9$ single-task environments from the OGBench suite \citep{ogbench_park2025} and $8$ tasks from Robomimic \citep{robomimic_mandlekar2021}. In all experiments, we use $4$ random seeds ($\pm$ denotes standard deviations)
and use the \textit{same} state-action value function trained via implicit Q-learning for both methods \citep{iql_kostrikov2022}, which notably does not require a policy in the loop to learn and is therefore independent of the extraction method that is used downstream. For AWR, we sweep over $1/\beta$ in the set of $\{1, 3, 10, 30\}$, and for CFGRL, we sweep over $w \in \{1, 1.25, 1.5, 2, 3\}$. For DDPG we sweep over a BC regularization terms $\alpha \in \{1, 3, 10\}$ and for FQL we sweep over $\alpha \in \{0.3, 1, 3, 10, 100\}$. Results are presented in \Cref{table:aggregation}. On Robomimic we only evaluate diffusion-based methods. On a strong majority of tasks, CFGRL is competitive to and even surpasses baselines, while not requiring any regularization tuning during training.

\section{CFGRL unlocks hidden gains in goal-conditioned behavioral cloning}
\label{section:gcbc}

While the overall CFGRL framework is general, it is particularly appealing in the special-case of goal-conditioned RL, where it is common to use a simple (though crude) approximation that bypasses the need for a learned value estimator \citep{ocbc_eysenbach2022, gcsl_ghosh2021, susie_black2024}. In such settings, the objective to find a goal-conditioned policy $\pi(\pl{a} \mid \pl{s}, \pl{g}): \gS \times \gS \to \Delta(\gA)$ that maximizes likelihood of reaching the goal:
\begin{equation}
    J(\pi) = \E_{\tau \sim p(\pl{\tau} \mid \pi),\ g \sim p(\pl{g})}\sum_t \gamma^t \delta_g(s_t),
    \label{eq:goal-reaching}
\end{equation}
where $p(\pl{g}) \in \Delta(\gS)$ is a goal distribution
and reward is $\delta_g$, i.e. the Dirac delta ``function'' at $g$.\footnote{This ``function'' is well-defined in a discrete state space, but requires a measure-theoretic formulation~\citep{fb_touati2021} to be well-defined in a continuous state space, which we omit for simplicity.}
While in principle, we can optimize the above objective with a full RL procedure, an often-used simplification is to perform goal-conditioned \textit{behavioral cloning} (GCBC), which maximizes:
\begin{equation}
    J_\mathrm{GCBC}(\theta) = \E_{\substack{(s_t, a_t) \sim \gD,\ \Delta \sim \mathrm{Geom}(1-\gamma)}}[\log \pi_\theta(a_t \mid s_t, s_{t+\Delta})], \label{eq:gcbc_obj}
\end{equation}
where $\mathrm{Geom}(1-\gamma)$ denotes the Geometric distribution with parameter $1 - \gamma$,
and we often denote $s_{t+\Delta}$ as $g$.
GCBC can be seen as a special case of conditional behavioral cloning methods that filter for actions that empirically reach a future goal.
While the GCBC objective above is simple,
it does not converge to the optimal goal-reaching policy, especially when the dataset is suboptimal~\citep{ocbc_eysenbach2022, ocbc_ghugare2024}

\textbf{Generalizing past na\"ive GCBC.}
Based on our CFGRL framework in \Cref{section:operator},
we now introduce a method to potentially further improve GCBC policies
\emph{without training value functions}.
The key insight is that, since CFGRL enables one step of policy improvement over the base policy,
applying CFGRL with guidance on the goal $g$ will produce a policy that can be better than the standard GCBC policy.
While this policy is still not as optimal as the full GCRL solution (which in general would require many steps of policy improvement),
prior work has shown that even one step of policy improvement can lead to significant performance gains in a range of RL settings~\citep{onestep_brandfonbrener2021, crl_eysenbach2022}.
Indeed, our experiments will confirm that CFGRL enables a significant increase in performance over standard GCBC.

We begin by noting that the GCBC policy that maximizes \Cref{eq:gcbc_obj} is given as follows~\citep{crl_eysenbach2022}:
\begin{equation}
    \pi(a \mid s, g) = \frac{\hat \pi(a \mid s)p^\gamma(g \mid s, a)}{p^\gamma(g \mid s)} \propto  \hat \pi(a \mid s) \; Q_{\hat \pi}(s, a, g), \label{eq:gcbc-factored}
\end{equation}
where $p^\gamma$ denotes the distribution induced by the Geometric goal sampling procedure.

Next, we note that the second factor in \Cref{eq:gcbc-factored} satisfies the conditions of \cref{thm:reweight} (i.e., a bounded, non-negative, non-decreasing function in $A_{\hat{\pi}}(s, a)$).\footnote{Specifically, we set $f$ as $Q_{\hat \pi}(s, a, g)$, which is a bounded, non-negative, non-decreasing function of $A_{\hat \pi}(s, a, g)$. This can also be understood as defining $p(o \mid g,s,a) \propto p^\gamma(g \mid s,a)$, i.e., an action's optimality is proportional to the likelihood of reaching the goal in the discounted future.}
Thus, we can invoke \cref{thm:reweight} to achieve a policy improvement.
Specifically,
the resulting policy $\pi(a \mid s) \propto \hat{\pi}(a \mid s) \; p(g \mid s,a)^w$ under any exponent $w \geq 0$ will result in a policy improvement over the (unconditional) BC policy,
and we can sample from the attenuated second factor via guidance:
\begin{equation}
\nabla_a \log \hat{\pi}(a \mid s) +  w \;(\nabla_a \log \pi(a \mid s,g) - \nabla_a \log \hat{\pi}(a \mid s)).
\label{eq:gcbc-guidance}
\end{equation}
The above CFGRL interpretation reveals a simple recipe for improving beyond the GCBC policy. Specifically, naive GCBC results in a product policy that implicitly assumes a weighting of $w=1$. Guidance allows us to instead consider $w>1$, potentially leading to further improved performance.

We emphasize the practical benefits of this setup---improvement can be gained for ``free'' relative to a standard GCBC setup. The components of \cref{eq:gcbc-guidance} are simply the original goal-conditioned BC policy $\pi(\pl{a} \mid \pl{s}, \pl{g})$ along with an unconditional BC policy $\hat{\pi}(\pl{a} \mid \pl{s})$. We do \emph{not} require any additional techniques such as training an explicit value function, or sampling further on-policy actions.

\subsection{Experimental results}

\begin{figure}
    \centering
    \includegraphics[width=\textwidth]{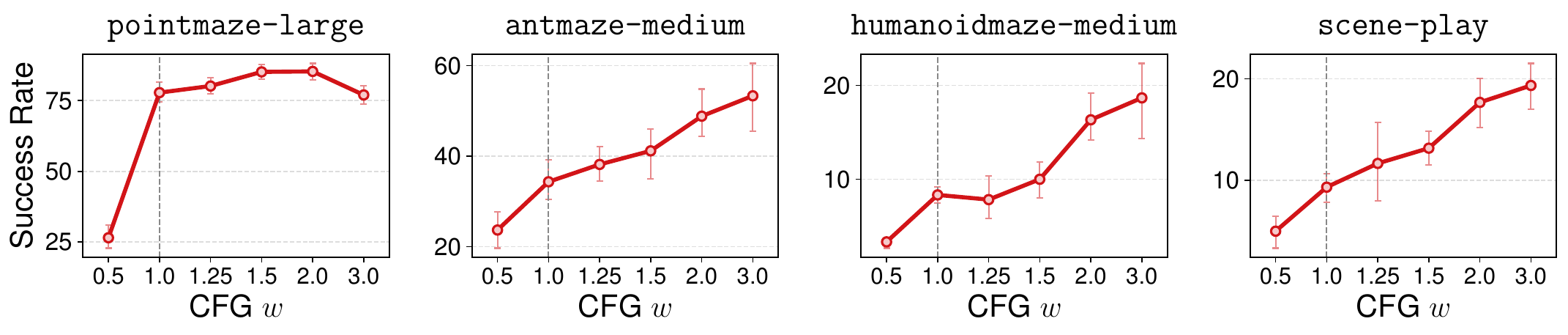}
    \vspace*{-4ex}
    \caption{
    \footnotesize
    \textbf{CFGRL can extrapolate \textit{beyond} the GCBC policy, unlocking further performance gains.} \newline In fact, GCBC is implicitly a special case of the CFGRL policy where $w=1$. By instead considering $w>1$, the resulting policy can be an improvement over the original. We show that performance generally increases with $w$ on a range of environments.
    }
    \vspace{-10pt}
    \label{fig:gcbc}
\end{figure}

\textbf{Tasks.}
We empirically verify effectiveness over $17$ state-based and $7$ pixel-based goal-conditioned RL tasks from the OGBench task suite~\citep{ogbench_park2025} (\Cref{fig:ogbench}).
These tasks span a variety of robotic navigation and manipulation domains,
including whole-body humanoid control, maze navigation, sequential object manipulation, and combinatorial puzzle solving.
Among them, the tasks prefixed with ``\texttt{visual-}'' require image-based control.

\textbf{Methods.}
In this experiment, we consider four imitation learning baselines that do not involve value learning
(recall that CFGRL also does not train a value function):
(1) BC, (2) flow BC~\citep{diffusionpolicy_chi2023}, (3) goal-conditioned BC (GCBC)~\citep{gcsl_ghosh2021}, and (4) flow GCBC.
BC trains an (unconditional) behavioral cloning policy,
and GCBC trains a goal-conditioned policy with \Cref{eq:gcbc_obj}.
Flow BC and flow GCBC maximize the same objectives, but with flow policies~\citep{diffusionpolicy_chi2023, pi0_black2024}.
In the CFGRL framework, flow BC corresponds to CFGRL with $w = 0$ and flow GCBC corresponds to $w = 1$.

In addition to the five ``flat'' methods above, we also consider \emph{hierarchical} behavioral cloning~\citep{ril_gupta2019, play_lynch2019} on state-based tasks,
where we train both a high-level policy $\pi^h(\pl{\ell} \mid \pl{s}, \pl{g}): \gS \times \gS \to \Delta(\gS)$ that outputs subgoals $\ell$
and a low-level policy $\pi^\ell(\pl{a} \mid \pl{s}, \pl{\ell}): \gS \times \gS \to \Delta(\gA)$ that takes subgoals and outputs actions.
In this setting, we can apply CFGRL to each level's GCBC objective to enhance the optimality of both policies.
We call this variant hierarchical CFGRL (HCFGRL).
As baselines, we consider hierarchical GCBC~\cite{ril_gupta2019, play_lynch2019} and flow hierarchical GCBC,
which trains Gaussian and flow policies, respectively.

\begin{table*}[t!]
\centering
\caption{
\footnotesize
\textbf{Improving on GCBC.}
CFGRL consistently improves performance over GCBC across the board.
Numbers at or above the $95\%$ of the best performance in each category are boldfaced, as in OGBench~\citep{ogbench_park2025}.
}
\scalebox{0.73}
{
\begin{tabular}{lcccccccc}
\toprule
\multicolumn{1}{c}{} & \multicolumn{5}{c}{\texttt{Flat Policies}} & \multicolumn{3}{c}{\texttt{Hierarchical Policies}} \\
\cmidrule(lr){2-6} \cmidrule(lr){7-9}
\texttt{Task} & \texttt{BC} & \texttt{Flow BC} & \texttt{GCBC} & \texttt{Flow GCBC} & \texttt{\color{myred}CFGRL} & \texttt{HGCBC} & \texttt{Flow HGCBC} & \texttt{\color{myred}HCFGRL} \\
\midrule

\texttt{pointmaze-medium-navigate} & $0$ {\tiny $\pm 0$} & $1$ {\tiny $\pm 2$} & $9$ {\tiny $\pm 5$} & $66$ {\tiny $\pm 4$} & $\mathbf{77}$ {\tiny $\pm 6$} & $0$ {\tiny $\pm 0$} & $57$ {\tiny $\pm 7$} & $\mathbf{63}$ {\tiny $\pm 5$} \\
\texttt{pointmaze-large-navigate} & $0$ {\tiny $\pm 0$} & $0$ {\tiny $\pm 0$} & $25$ {\tiny $\pm 9$} & $\mathbf{74}$ {\tiny $\pm 7$} & $\mathbf{77}$ {\tiny $\pm 5$} & $0$ {\tiny $\pm 0$} & $\mathbf{75}$ {\tiny $\pm 6$} & $57$ {\tiny $\pm 8$} \\
\texttt{pointmaze-giant-navigate} & $0$ {\tiny $\pm 0$} & $0$ {\tiny $\pm 0$} & $2$ {\tiny $\pm 2$} & $4$ {\tiny $\pm 2$} & $\mathbf{30}$ {\tiny $\pm 10$} & $0$ {\tiny $\pm 0$} & $6$ {\tiny $\pm 5$} & $\mathbf{18}$ {\tiny $\pm 8$} \\
\texttt{pointmaze-teleport-navigate} & $0$ {\tiny $\pm 0$} & $0$ {\tiny $\pm 1$} & $24$ {\tiny $\pm 7$} & $\mathbf{46}$ {\tiny $\pm 4$} & $41$ {\tiny $\pm 8$} & $6$ {\tiny $\pm 5$} & $\mathbf{37}$ {\tiny $\pm 10$} & $\mathbf{38}$ {\tiny $\pm 13$} \\
\texttt{antmaze-medium-navigate} & $13$ {\tiny $\pm 2$} & $15$ {\tiny $\pm 6$} & $25$ {\tiny $\pm 8$} & $42$ {\tiny $\pm 9$} & $\mathbf{53}$ {\tiny $\pm 12$} & $45$ {\tiny $\pm 7$} & $67$ {\tiny $\pm 8$} & $\mathbf{90}$ {\tiny $\pm 5$} \\
\texttt{antmaze-large-navigate} & $11$ {\tiny $\pm 4$} & $5$ {\tiny $\pm 2$} & $20$ {\tiny $\pm 4$} & $22$ {\tiny $\pm 5$} & $\mathbf{24}$ {\tiny $\pm 10$} & $45$ {\tiny $\pm 4$} & $61$ {\tiny $\pm 5$} & $\mathbf{78}$ {\tiny $\pm 4$} \\
\texttt{antmaze-giant-navigate} & $0$ {\tiny $\pm 0$} & $0$ {\tiny $\pm 0$} & $0$ {\tiny $\pm 0$} & $0$ {\tiny $\pm 0$} & $\mathbf{1}$ {\tiny $\pm 1$} & $8$ {\tiny $\pm 5$} & $14$ {\tiny $\pm 5$} & $\mathbf{38}$ {\tiny $\pm 7$} \\
\texttt{antmaze-teleport-navigate} & $4$ {\tiny $\pm 2$} & $3$ {\tiny $\pm 2$} & $19$ {\tiny $\pm 5$} & $24$ {\tiny $\pm 7$} & $\mathbf{35}$ {\tiny $\pm 9$} & $35$ {\tiny $\pm 6$} & $38$ {\tiny $\pm 5$} & $\mathbf{50}$ {\tiny $\pm 5$} \\
\texttt{humanoidmaze-medium-navigate} & $3$ {\tiny $\pm 3$} & $2$ {\tiny $\pm 2$} & $6$ {\tiny $\pm 3$} & $8$ {\tiny $\pm 3$} & $\mathbf{19}$ {\tiny $\pm 6$} & $13$ {\tiny $\pm 3$} & $23$ {\tiny $\pm 4$} & $\mathbf{64}$ {\tiny $\pm 10$} \\
\texttt{humanoidmaze-large-navigate} & $0$ {\tiny $\pm 1$} & $0$ {\tiny $\pm 0$} & $2$ {\tiny $\pm 1$} & $1$ {\tiny $\pm 2$} & $\mathbf{3}$ {\tiny $\pm 2$} & $8$ {\tiny $\pm 5$} & $11$ {\tiny $\pm 2$} & $\mathbf{38}$ {\tiny $\pm 4$} \\
\texttt{antsoccer-arena-navigate} & $2$ {\tiny $\pm 2$} & $4$ {\tiny $\pm 1$} & $4$ {\tiny $\pm 3$} & $10$ {\tiny $\pm 5$} & $\mathbf{15}$ {\tiny $\pm 6$} & $8$ {\tiny $\pm 4$} & $16$ {\tiny $\pm 5$} & $\mathbf{37}$ {\tiny $\pm 5$} \\
\texttt{antsoccer-medium-navigate} & $0$ {\tiny $\pm 0$} & $0$ {\tiny $\pm 0$} & $\mathbf{6}$ {\tiny $\pm 4$} & $5$ {\tiny $\pm 3$} & $5$ {\tiny $\pm 3$} & $7$ {\tiny $\pm 2$} & $8$ {\tiny $\pm 3$} & $\mathbf{11}$ {\tiny $\pm 4$} \\
\texttt{cube-single-play} & $4$ {\tiny $\pm 2$} & $7$ {\tiny $\pm 5$} & $6$ {\tiny $\pm 2$} & $8$ {\tiny $\pm 2$} & $\mathbf{11}$ {\tiny $\pm 5$} & $12$ {\tiny $\pm 3$} & $26$ {\tiny $\pm 5$} & $\mathbf{46}$ {\tiny $\pm 6$} \\
\texttt{cube-double-play} & $1$ {\tiny $\pm 1$} & $2$ {\tiny $\pm 1$} & $2$ {\tiny $\pm 1$} & $\mathbf{3}$ {\tiny $\pm 1$} & $2$ {\tiny $\pm 2$} & $2$ {\tiny $\pm 1$} & $21$ {\tiny $\pm 7$} & $\mathbf{42}$ {\tiny $\pm 5$} \\
\texttt{scene-play} & $5$ {\tiny $\pm 2$} & $5$ {\tiny $\pm 3$} & $4$ {\tiny $\pm 3$} & $15$ {\tiny $\pm 5$} & $\mathbf{19}$ {\tiny $\pm 4$} & $7$ {\tiny $\pm 2$} & $14$ {\tiny $\pm 4$} & $\mathbf{18}$ {\tiny $\pm 5$} \\
\texttt{puzzle-3x3-play} & $\mathbf{3}$ {\tiny $\pm 2$} & $2$ {\tiny $\pm 2$} & $\mathbf{3}$ {\tiny $\pm 2$} & $1$ {\tiny $\pm 1$} & $2$ {\tiny $\pm 2$} & $\mathbf{3}$ {\tiny $\pm 2$} & $1$ {\tiny $\pm 1$} & $2$ {\tiny $\pm 2$} \\
\texttt{puzzle-4x4-play} & $\mathbf{0}$ {\tiny $\pm 0$} & $\mathbf{0}$ {\tiny $\pm 1$} & $\mathbf{0}$ {\tiny $\pm 0$} & $\mathbf{0}$ {\tiny $\pm 0$} & $\mathbf{0}$ {\tiny $\pm 0$} & $\mathbf{0}$ {\tiny $\pm 0$} & $\mathbf{0}$ {\tiny $\pm 0$} & $\mathbf{0}$ {\tiny $\pm 0$} \\
\midrule
\texttt{visual-antmaze-medium-navigate} & $14$ {\tiny $\pm 5$} & $7$ {\tiny $\pm 1$} & $11$ {\tiny $\pm 5$} & $19$ {\tiny $\pm 4$} & $\mathbf{23}$ {\tiny $\pm 4$} & - & - & - \\
\texttt{visual-antmaze-large-navigate} & $6$ {\tiny $\pm 4$} & $1$ {\tiny $\pm 1$} & $5$ {\tiny $\pm 2$} & $3$ {\tiny $\pm 1$} & $\mathbf{11}$ {\tiny $\pm 2$} & - & - & - \\
\texttt{visual-cube-single-play} & $1$ {\tiny $\pm 1$} & $1$ {\tiny $\pm 1$} & $1$ {\tiny $\pm 2$} & $13$ {\tiny $\pm 7$} & $\mathbf{37}$ {\tiny $\pm 9$} & - & - & - \\
\texttt{visual-cube-double-play} & $0$ {\tiny $\pm 0$} & $0$ {\tiny $\pm 0$} & $0$ {\tiny $\pm 0$} & $\mathbf{10}$ {\tiny $\pm 4$} & $7$ {\tiny $\pm 5$} & - & - & - \\
\texttt{visual-scene-play} & $5$ {\tiny $\pm 1$} & $7$ {\tiny $\pm 2$} & $5$ {\tiny $\pm 2$} & $25$ {\tiny $\pm 4$} & $\mathbf{40}$ {\tiny $\pm 8$} & - & - & - \\
\texttt{visual-puzzle-3x3-play} & $\mathbf{1}$ {\tiny $\pm 2$} & $0$ {\tiny $\pm 1$} & $0$ {\tiny $\pm 1$} & $0$ {\tiny $\pm 1$} & $0$ {\tiny $\pm 1$} & - & - & - \\
\texttt{visual-puzzle-4x4-play} & $\mathbf{0}$ {\tiny $\pm 0$} & $\mathbf{0}$ {\tiny $\pm 0$} & $\mathbf{0}$ {\tiny $\pm 0$} & $\mathbf{0}$ {\tiny $\pm 0$} & $\mathbf{0}$ {\tiny $\pm 0$} & - & - & - \\

\bottomrule
\end{tabular}
}
\vspace{-1em}
\label{table:gcbc}
\end{table*}

\textbf{Results.} \Cref{table:gcbc} presents evaluation results averaged over $8$ seeds for state-based tasks and $4$ seeds for pixel-based tasks,
denoting standard deviations after the $\pm$ symbols.
The results suggest that CFGRL consistently outperforms all four baselines on most of the tasks,
even with a single fixed value of the guidance strength ($w=3$).
Notably, on some tasks (e.g., \texttt{pointmaze-giant} and \texttt{visual-cube-single}),
CFGRL achieves more than $3 \times$ the success of the strongest baseline.
We again emphasize that this improvement is achieved simply by contrasting the prior and GCBC policies, \emph{without} training a value function.
As in \Cref{section:offline}, we also measure how the performance of CFGRL varies with different values of CFG weights $w$.
We present the results on four tasks in \Cref{fig:gcbc} (see \Cref{fig:gcbc_full} for the full results),
which shows that the performance generally improves as $w$ increases.

\section{Discussion and conclusion}

In this work, we introduced a principled connection between diffusion guidance and policy improvement in RL. Using this connection, we derive a simple framework that combines the simplicity of existing generative modeling objectives with the policy improvement capabilities of RL. We then instantiate this framework for policy extraction in offline RL, and for improving over goal-conditioned BC (GCBC) in the goal-conditioned setting. We show that CFGRL improves over common approaches in the offline RL setting without needing a large hyperparameter sweep, and achieves a substantial improvement over GCBC in the goal-conditioned setting, while maintaining the simplicity of these prior methods.

\textbf{Limitations.}
Our theoretical results are based on idealized assumptions (e.g., ground-truth advantages and exact sampling from the product policy)
and we have not verified CFGRL outside simulated benchmarks.
We leave addressing these points for future work.
Our method does not claim to replace full RL procedures---we assume a given value function and do not make any prescriptions about how to train it.
In our experiments, CFGRL takes the place of prior supervised learning methods for policy extraction, maintaining their simplicity and stability. 
By itself, CFGRL does not represent a state-of-the-art RL algorithm, but rather an additional tool in the algorithm designer's toolbox that can take the place of policy extraction methods such as AWR, DDPG, and FQL, as well as a theoretical connection that we hope will inspire future work.

\section*{Acknowledgments}
We thank Zhiyuan Zhou for helping with experiments and Jie Ruan for pointing out an issue in a theorem in a previous version of the draft.
This work was supported in part by an NSF Fellowship for KF, under grant No. DGE 2146752, and by the Korea Foundation for Advanced Studies (KFAS) for SP,
and AFOSR FA9550-22-1-0273, and ONR N00014-22-1-2773 for SL. Any opinions, findings, and conclusions or recommendations expressed in this material are those of the author(s) and do not necessarily reflect the views of the NSF. PA holds concurrent appointments as a Professor at UC Berkeley and as an Amazon Scholar. This paper describes work performed at UC Berkeley and is not associated with Amazon. This research used the Savio computational cluster resource provided by the Berkeley Research Computing program at UC Berkeley.

{\footnotesize
\bibliographystyle{plainnat}
\bibliography{neurips_2026}
}

\appendix

\section{Theoretical results}

\begin{lemma}[Chebyshev's sum inequality for probability measures] \label{lem:chebyshev}
For any probability measure $\mu$ on $\sR$ and any bounded, measurable, non-decreasing functions $g, h: \sR \to \sR$,
\begin{align}
    \int_\sR g(x)h(x) \mu(\de x) \geq \int_\sR g(x) \mu(\de x) \int_\sR h(x) \mu(\de x).
\end{align}
\end{lemma}
\begin{proof}
Since $g$ and $h$ are non-decreasing, the signs of $g(y) - g(z)$ and $h(y) - h(z)$ are the same for any $y, z \in \sR$.
Hence, we have
\begin{align}
    0 &\leq \int_{\sR \times \sR} (g(y) - g(z)) (h(y) - h(z)) (\mu \otimes \mu) (\de y, \de z) \\
    &= \int_\sR \left( \int_\sR (g(y)h(y) + g(z)h(z) - g(y)h(z) - g(z)h(y)) \mu(\de y) \right) \mu(\de z) \\
    &= 2 \int_\sR g(x)h(x) \mu(\de x) - 2 \int_\sR g(x) \mu(\de x) \int_\sR h(x) \mu(\de x),
\end{align}
from which the conclusion follows,
where $\mu \otimes \mu$ denotes the product measure of $\mu$ and itself,
and we use Fubini's theorem in the second line.
\end{proof}

\begin{lemma} \label{lem:imp_lemma}
Let $s \in \gS$ be a state, $\pi, \hat \pi: \gS \to \Delta(\gA)$ be policies, and $f: \sR \to \sR$ be a bounded, measurable, non-negative, non-decreasing function.
Suppose that $\pi(a \mid s) = f(A_{\hat{\pi}} (s, a)) \hat \pi(a \mid s)$ and $\E_{a \sim \hat \pi(\cdot \mid s)}[f(A_{\hat \pi}(s, a))] = 1$.
Then,
\begin{align}
    \E_{a \sim \pi(\cdot \mid s)}[Q_{\hat \pi}(s, a)] \geq V_{\hat \pi}(s). \label{eq:imp_lemma}
\end{align}
\end{lemma}
\begin{proof}
To apply \Cref{lem:chebyshev}, we first rewrite the left-hand side of \Cref{eq:imp_lemma} using probability measures as follows:
\begin{align}
    \E_{a \sim \pi(\cdot \mid s)}[Q_{\hat \pi}(s, a)]
    &= \int_\gA Q_{\hat{\pi}}(s, a) \pi_s(\de a) \\
    &= \int_\gA Q_{\hat{\pi}}(s, a) f(A_{\hat{\pi}}(s, a)) \hat \pi_s(\de a) \\
    &= \int_\gA Q_{\hat{\pi}}(s, a) f(Q_{\hat{\pi}}(s, a) - V_{\hat{\pi}}(s)) \hat \pi_s(\de a),
\end{align}
where $\pi_s$ and $\hat \pi_s$ denote the probability measures corresponding to
the distributions $\pi(\cdot \mid s)$ and $\hat \pi(\cdot \mid s)$, respectively.
Then,
\begin{align}
    \E_{a \sim \pi(\cdot \mid s)}[Q_{\hat \pi}(s, a)]
    &= \int_\gA Q_{\hat{\pi}}(s, a) f(Q_{\hat{\pi}}(s, a) - V_{\hat{\pi}}(s)) \hat \pi_s(\de a)  \\
    &= \int_\sR q f(q - V_{\hat \pi}(s)) \lambda(\de q) \\
    &\geq \left( \int_\sR q \lambda(\de q) \right) \left( \int_\sR f(q - V_{\hat{\pi}}(s)) \lambda(\de q) \right) \\
    &= \left( \int_\gA Q_{\hat{\pi}}(s, a) \hat \pi_s(\de a) \right) \left( \int_\gA f(Q_{\hat{\pi}}(s, a) - V_{\hat{\pi}}(s)) \hat \pi_s(\de a) \right) \\
    &= \left( \int_\gA Q_{\hat{\pi}}(s, a) \hat \pi_s(\de a) \right) \left( \int_\gA f(A_{\hat{\pi}}(s, a)) \hat \pi_s(\de a) \right) \\
    &= V_{\hat{\pi}}(s) \E_{a \sim \hat \pi(\cdot \mid s)}[f(A_{\hat \pi}(s, a))] \\
    &= V_{\hat{\pi}}(s),
\end{align}
where $\lambda$ denotes the pushforward measure of $\hat \pi_s$ by $Q_{\hat{\pi}}(s, \cdot)$,
and we use \Cref{lem:chebyshev} in the third line
with $g(x) = x$ and $h(x) = f(x - V_{\hat{\pi}}(s))$,
both of which are non-decreasing.
\end{proof}

\begin{lemma}[Policy improvement theorem for stochastic policies~\citep{rl_sutton2005, rl_silva2023}] \label{lem:pit}
For any policies $\pi$ and $\hat \pi$ satisfying $\E_{a \sim \pi(\cdot \mid s)}[Q_{\hat \pi}(s, a)] \geq V_{\hat \pi}(s)$ for all $s \in \gS$,
\begin{align}
    J(\pi) \geq J(\hat \pi).
\end{align}
\end{lemma}
\begin{proof}
    This is a straightforward generalization of the policy improvement theorem to stochastic policies.
    See Section 4.2 of \citet{rl_sutton2005} and Theorem 3 of \citet{rl_silva2023}.
\end{proof}

\begin{theorem}[Policy improvement by reweighting] \label{thm:reweight}
Let $\pi, \hat \pi: \gS \to \Delta(\gA)$ be policies and $f: \sR \to \sR$ be a bounded, measurable, non-negative, non-decreasing function.
Suppose that $\pi$ satisfies $\pi(a \mid s) \propto f(A_{\hat{\pi}} (s, a)) \hat \pi(a \mid s)$.
Then,
\begin{align}
    J(\pi) \geq J(\hat \pi).
\end{align}
\end{theorem}
\begin{proof}
Fix $s \in \gS$.
Let $\pi(a \mid s) = f(A_{\hat \pi}(s, a))\hat \pi(a \mid s) / Z(s)$,
where the normalization function $Z: \gS \to \sR$ is defined as
\begin{align}
    Z(s) = \int_\gA f(A_{\hat \pi}(s, a)) \hat \pi_s(\de a).
\end{align}
Then, we have
\begin{align}
1 &= \int_\gA f(A_{\hat \pi}(s, a))/Z(s) \hat \pi_s(\de a) \\
&= \E_{a \sim \hat \pi(\cdot \mid s)}[f(A_{\hat \pi}(s, a))/Z(s)].
\end{align}
Defining $g = f / Z(s)$, we get $\E_{a \sim \hat \pi(\cdot \mid s)}[g(A_{\hat \pi}(s, a))] = 1$.
Since $f$ is non-negative and non-decreasing, so is $g$,
and the conclusion directly follows from \Cref{lem:imp_lemma}
(with $\pi(a \mid s) = g(A_{\hat{\pi}} (s, a)) \hat \pi(a \mid s)$)
and \Cref{lem:pit}.
\end{proof}

\clearpage

\section{Additional results}

\begin{figure}[t!]
    \centering
    \includegraphics[width=\textwidth]{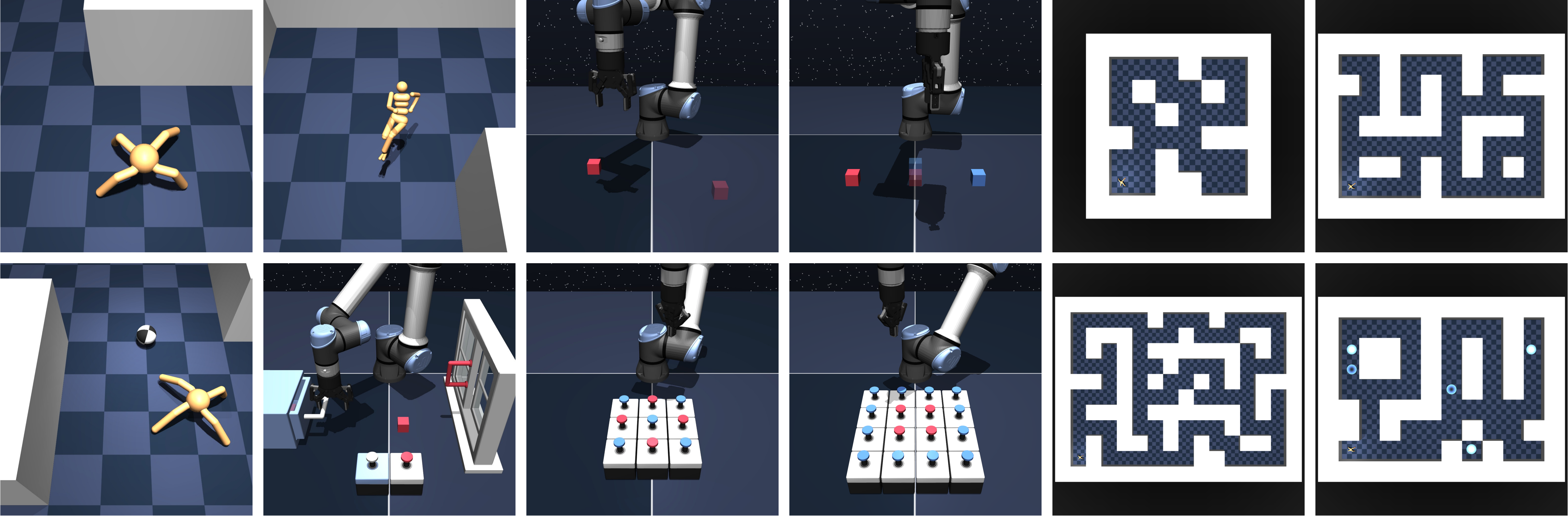}
    \caption{
    \footnotesize
    \textbf{OGBench environments.}
    }
    \label{fig:ogbench}
\end{figure}

\begin{figure}[t!]
    \centering
    \includegraphics[width=\textwidth]{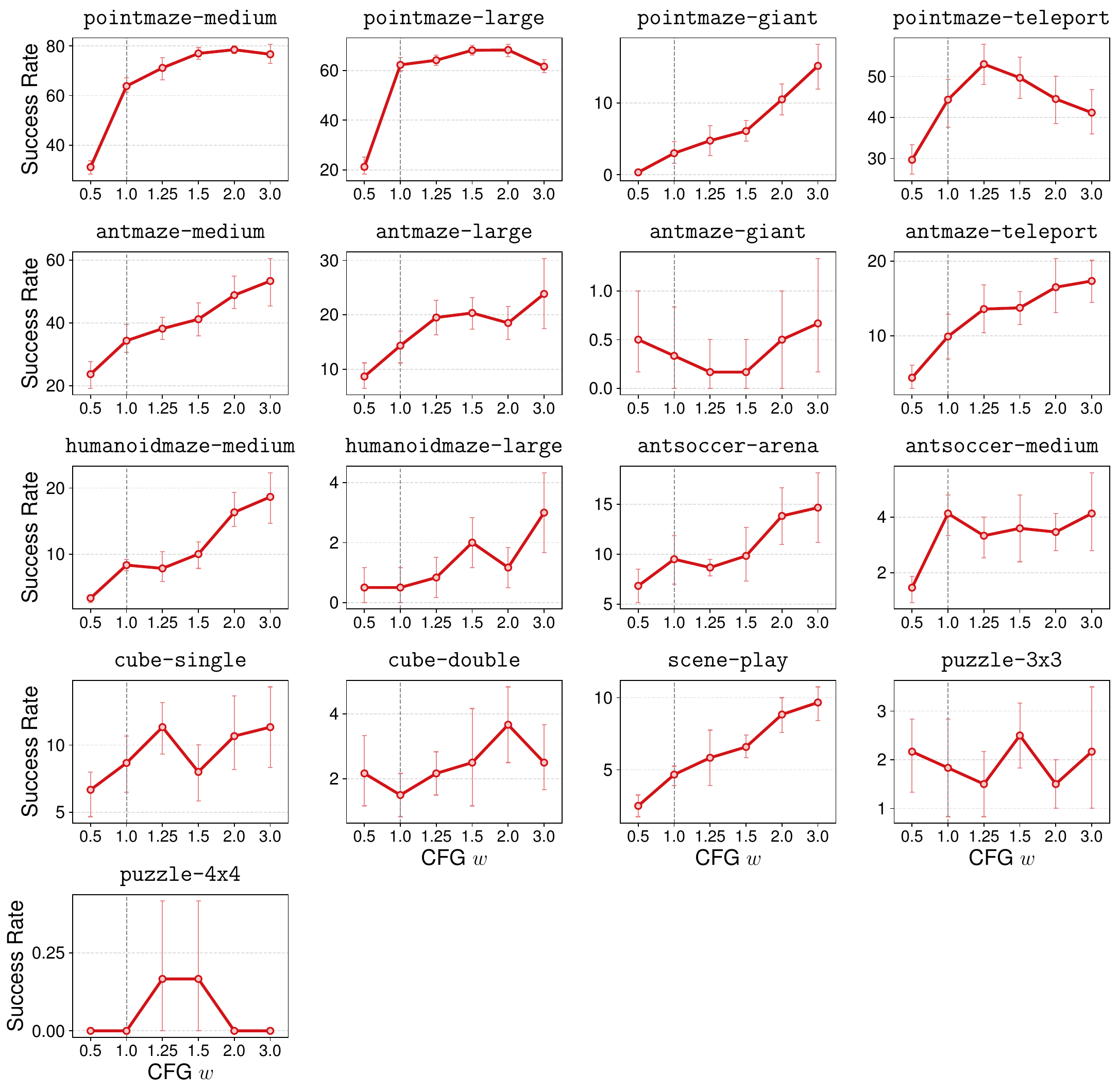}
    \caption{
    \footnotesize
    \textbf{Full ablation results on CFG weight $\bm{w}$.}
    The performance of CFGRL generally improves as the CFG weight increases.
    }
    \label{fig:gcbc_full}
\end{figure}

\begin{figure}[t!]
    \centering
    \includegraphics[width=\textwidth]{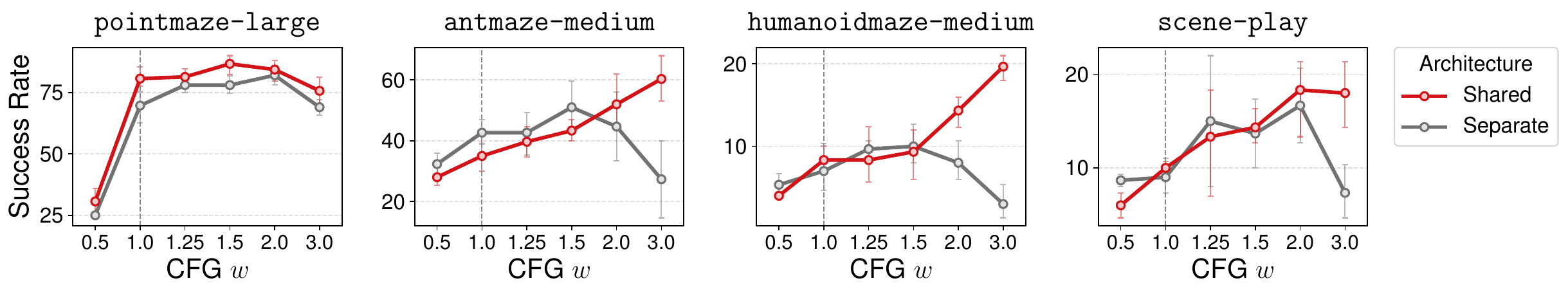}
    \caption{
    \footnotesize
    \textbf{Ablation study on optimality conditioning.}
    Shared policies lead to better performance and extrapolation than separate policies,
    likely because the former shares representations.
    }
    \label{fig:gcbc_arch}
\end{figure}

\textbf{Environments.} \Cref{fig:ogbench} illustrates OGBench tasks.

\textbf{Ablation study on the CFG weight $\bm{w}$.}
We present the full ablation study on the CFG weight $w$ across all $17$ state-based OGBench tasks in \Cref{fig:gcbc_full}.
The results show that the performance improves as the CFG weight increases,
although it sometimes declines beyond a certain point,
likely because the policy deviates too far from the data distribution.

\textbf{Ablation study on optimality conditioning.}
When modeling an optimality-conditioned policy $\pi(\pl{a} \mid \pl{s}, \pl{o})$ with $o \in \{\varnothing, 0, 1\}$,
we can either have separate networks for each $o$ value, or share the same network with a learnable optimality embedding.
We choose the latter in our experiments, and present an ablation study in \Cref{fig:gcbc_arch}.
The results suggest that the shared architecture generally works and extrapolates better than the separate one.
We believe this is likely because extrapolation benefits from shared representations.

\section{Implementation details}

We implement CFGRL on top of the reference implementations provided by OGBench~\citep{ogbench_park2025}.
Each experiment in this work takes no more than $4$ hours on a single A5000 GPU.
Our implementation and example training commands are available at \url{https://github.com/kvfrans/cfgrl}.

\textbf{Tasks.}
In \Cref{section:offline}, 
we employ $8$ tasks from the Robomimic benchmark~\citep{robomimic_mandlekar2021}
and $9$ single-task (\texttt{singletask}) variants from the OGBench suite~\citep{ogbench_park2025}.
In \Cref{section:gcbc}, we employ the \texttt{oraclerep} variant of OGBench tasks
to remove confounding factors related to goal representation learning,
where this variant provides ground-truth goal representations
(e.g., in \texttt{antmaze}, a goal is specified by only the $x$-$y$ position, as opposed to the full $29$-dimensional state including proprioceptive information).

\textbf{Methods and hyperparameters.}
For baselines, we follow the original implementations and hyperparameters whenever possible~\citep{iql_kostrikov2022, ogbench_park2025, fql_park2025}.
For GCBC methods in \Cref{section:gcbc}, we sample goals uniformly from future states, as in the original implementation in OGBench~\citep{ogbench_park2025}.
This can be viewed as an approximation of geometric sampling with a high $\gamma$.
We present the full list of the hyperparameters in \Cref{table:hyp_offline_ogbench,table:hyp_gcbc}.

\clearpage

\begin{table}[t]
\caption{
\footnotesize
\textbf{Hyperparameters for OGBench and Robomimic offline RL experiments (\Cref{table:aggregation}).}
}
\vspace{5pt}
\label{table:hyp_offline_ogbench}
\begin{center}
\scalebox{0.8}
{
\begin{tabular}{ll}
    \toprule
    \textbf{Hyperparameter} & \textbf{Value} \\
    \midrule
    Learning rate & $0.0003$ \\
    Optimizer & Adam~\citep{adam_kingma2015} \\
    Gradient steps & $500000$ \\
    Minibatch size & $256$ \\
    MLP dimensions & $[512, 512, 512, 512]$ \\
    Nonlinearity & GELU~\citep{gelu_hendrycks2016} \\
    Target network smoothing coefficient & $0.005$ \\
    Discount factor $\gamma$ & $0.99$ (default), $0.995$ (\texttt{antmaze-giant}, \texttt{humanoidmaze}, \texttt{antsoccer}) \\
    Flow steps & $16$ \\
    Flow time sampling distribution & $\mathrm{Unif}([0, 1])$ \\
    IQL expectile & $0.9$ \\
    AWR $1/\beta$ and CFGRL $w$ & Swept per task \\
    \bottomrule
\end{tabular}
}
\end{center}
\end{table}

\begin{table}[t]
\caption{
\footnotesize
\textbf{Hyperparameters for GCBC experiments (\Cref{table:gcbc}).}
}
\vspace{5pt}
\label{table:hyp_gcbc}
\begin{center}
\scalebox{0.8}
{
\begin{tabular}{ll}
    \toprule
    \textbf{Hyperparameter} & \textbf{Value} \\
    \midrule
    Learning rate & $0.0003$ \\
    Optimizer & Adam~\citep{adam_kingma2015} \\
    Gradient steps & $1000000$ \\
    Minibatch size & $1024$ (states), $256$ (pixels) \\
    MLP dimensions & $[512, 512, 512, 512]$ \\
    Nonlinearity & GELU~\citep{gelu_hendrycks2016} \\
    Image augmentation probability & $0.5$ \\
    Flow steps & $16$ \\
    Flow time sampling distribution & $\mathrm{Unif}([0, 1])$ \\
    CFGRL $w$ & $3$ \\
    Subgoal steps for hierarchical BC & $25$ (default), $10$ (OGBench manipulation), $50$ (\texttt{humanoidmaze}) \\
    \bottomrule
\end{tabular}
}
\end{center}
\end{table}

\end{document}